\documentclass{article}
\usepackage{ijcai19}

\usepackage{multirow}
\usepackage{adjustbox}
\usepackage{array}
\usepackage{times}
\usepackage{soul}
\usepackage{url}
\usepackage[hidelinks]{hyperref}
\usepackage[utf8]{inputenc}
\usepackage[small]{caption}
\usepackage{graphicx}
\usepackage{amsmath}
\usepackage{booktabs}
\usepackage{times}% For figures
\usepackage{graphicx} % more modern
\usepackage{enumerate}
\usepackage{amssymb}
\usepackage{amsmath}
\usepackage{algorithm}
\usepackage{bm}
\usepackage{algpseudocode}
\usepackage{subcaption}
\usepackage{amsmath}
\usepackage{multirow}
\usepackage{amsthm}
\usepackage{graphicx,mathtools}

\usepackage{url}
\usepackage{caption,color}
\usepackage{pifont}

\usepackage[framed,numbered]{mcode}

\newtheorem{theorem}{Theorem}

\title{HDI-Forest: Highest Density Interval Regression Forest}
\author{
	Lin Zhu\footnote{Contact Author}\and
	Jiaxing Lu\and
	Yihong Chen\\
	\affiliations
	Ctrip Travel Network Technology Co., Limited.\\
	\emails
	\{zhulb, lujx, yihongchen\}@Ctrip.com
}

\begin{document}

\maketitle

\begin{abstract}	
By seeking the narrowest prediction intervals (PIs) that satisfy the specified coverage probability requirements, the recently proposed quality-based PI learning principle can extract high-quality PIs that better summarize the predictive certainty in regression tasks, and has been widely applied to solve many practical problems. Currently, the state-of-the-art quality-based PI estimation methods are based on deep neural networks or linear models. In this paper, we propose Highest Density Interval Regression Forest (HDI-Forest), a novel quality-based PI estimation method that is instead based on Random Forest. HDI-Forest does not require additional model training, and directly reuses the trees learned in a standard Random Forest model. By utilizing the special properties of Random Forest, HDI-Forest could efficiently and more directly optimize the PI quality metrics. Extensive experiments on benchmark datasets show that HDI-Forest significantly outperforms previous approaches, reducing the average PI width by over 20\% while achieving the same or better coverage probability.
\end{abstract}

\section{Introduction}

Let ${D}_{XY}$ be an unknown joint distribution over instances $x\in \mathcal{X}$ and responses $y\in \mathbb{R}$, where \textit{X}, \textit{Y} denote random variables, and \textit{x}, \textit{y} are their instantiations. A common goal shared by many predictive tasks is to infer certain properties of the conditional distribution ${{D}_{\left. Y \right|X}}$. For example, in a standard regression task, we are given a training set of $\left\{ \left( {{x}_{i}},{{y}_{i}} \right) \right\}_{i=1}^{n}$ sampled i.i.d. from ${D}_{XY}$ and a new test instance $x$ sampled from ${D}_{X}$, the goal is to predict $E\left( Y\left| X=x \right. \right)$, namely the conditional mean of \textit{Y} at $x$. Although estimation of the mean is highly useful in practice, it conveys no information about the predictive uncertainty, which can be very important for improving the reliability and robustness of the predictions \cite{khosravi2011lower,pmlr-v80-pearce18a}. 

Prediction intervals (PIs) is one of the most widely used tools for quantifying and representing the uncertainty of predictions. For a specified confidence level $\alpha$, the goal of PI construction is to estimate the $100\left( 1-\alpha  \right)\%$ interval $\left[ l,u \right]\in {{\mathbb{R}}^{2}}$ that will cover no less than $1-\alpha$ of the probability mass of ${{D}_{\left. Y \right|X=x}}$, namely:
\begin{equation}
\label{PI}
P\left( \left. l\le Y\le u \right|X=x \right)\ge 1-\alpha.
\end{equation}
PIs directly express uncertainty by providing lower and upper bounds for each prediction with specified coverage probability, and are more informative and useful for decision making than conditional means alone \cite{pmlr-v80-pearce18a,stine1985bootstrap}.

A plethora of techniques have been proposed in the literature for construction of PIs \cite{rosenfeld2018}. However, the majority of existing works only consider the coverage probability criteria (\ref{PI}), and yet ignore other crucial aspects of the elicited PIs \cite{khosravi2011lower}. In particular, there exists a fundamental trade-off between PI width and coverage probability, and (\ref{PI}) can always be trivially fulfilled by a large enough yet useless interval \cite{rosenfeld2018}. This motivates the development of \textit{quality-based} PI elicitation principle, which seeks the shortest PI that contains the required amount of probability \cite{pmlr-v80-pearce18a}. So far, quality-based PI estimation has been applied to solve many practical problems, such as the predictions of electronic price \cite{shrivastava2015prediction}, wind speed \cite{lian2016landslide}, and solar energy \cite{galvan2017multi}, etc. 

Although some promising results have been shown, existing approaches for quality-based PI construction still have some limitations. Firstly, most of these methods are built upon deep neural networks (DNNs) \cite{khosravi2011lower,pmlr-v80-pearce18a}, and yet currently the quality-based PI learning principle has primarily been applied to handle tabular data. Despite the tremendous success of DNNs for various domains such as image and text processing, it is known that for tabular data, tree-based ensembles such as Random Forest \cite{breiman2001random} and Gradient Boosting Decision Tree (GBDT) \cite{friedman2001greedy} often perform better, and are more widely used in practice \cite{klambauer2017self,huang2015telco,feng2018multi}. Secondly, quality-based PI learning objectives are generally non-convex, non-differentiable, and even discontinuous, and are thus difficult to optimize. Although existing methods partially solved this problem by optimizing continuous and differentiable surrogate functions instead \cite{pmlr-v80-pearce18a}, the overall predictive performance may be improved by resolving such a mismatch between the objective functions used in training and the final evaluation metrics used in testing.

Motivated by the above considerations, in this paper we propose Highest Density Interval Regression Forest (HDI-Forest), a novel quality-based PI estimation method based on Random Forest. HDI-Forest does not require additional model training, and directly reuses the trees learned in a standard Random Forest model. By utilizing the special properties of Random Forest introduced in previous works \cite{lin2006random,meinshausen2006quantile}, HDI-Forest could efficiently and more directly optimize the PI quality evaluation metrics. Experiments on benchmark datasets show that HDI-Forest significantly outperforms previous approaches in terms of PI quality metrics.

The rest of the paper is organized as follows. In Section 2 we review related work on prediction intervals. In Section 3, the mechanism of Random Forest is introduced, along with its interpretation as an approximate nearest neighbor method \cite{lin2006random}. Using this interpretation, HDI-Forest is introduced in Section 4 as a generalization of Random Forest. Encouraging numerical results for benchmark data sets are presented in Section 5.

\section{Related Work}

\subsection{Quantile-based PI Estimation}

For the pair of random variables $\left( X,Y \right)$, the \textit{conditional quantile} ${{Q}_{\tau }}\left( x \right)$ is the cut point that satisfies:
\begin{equation}
	P\left( Y\le {{Q}_{\tau }}\left( x \right)\left| X=x \right. \right)=\tau.
\end{equation}
It is easy to verify that the equal-tailed interval $\left[ l,u \right]=\left[ {{Q}_{\alpha /2}}\left( x \right),{{Q}_{1-\alpha /2}}\left( x \right) \right]$ is a valid solution to (\ref{PI}). Based on this insight, the classic approach for constructing PIs would first estimate ${{D}_{\left. Y \right|X=x}}$, and then estimate its quantiles as solutions \cite{rosenfeld2018}. If ${{D}_{\left. Y \right|X=x}}$ is assumed to have some parametric form (e.g., Gaussian), it is often possible to compute the desired quantiles in closed-form. However, explicit assumptions about the conditional distribution may be too restrictive for real-world data modeling \cite{sharpe1970robustness}, and considerable research has been devoted to tackle this limitation. For example, Quantile Regression \cite{koenker2001quantile} avoid explicit specification of the conditional distribution, and directly infer the conditional quantiles from data by minimizing an asymmetric variant of the absolute loss. On the other hand, re-sampling-based approaches would train a number of models on different re-sampled versions of the training dataset, where commonly used re-sampling techniques include leave-one-out \cite{steinberger2016leave} and bootstrap \cite{stine1985bootstrap}, and then use the point forecasts of trained models to estimate the conditional quantiles. Due to the need for training multiple models, a major disadvantage of re-sampling-based methods is the high computational cost for large-scale datasets \cite{rivals2000construction,khosravi2011lower}. A notable exception is quantile regression forest \cite{meinshausen2006quantile}, which estimates quantiles by using the sets of local weights generated by Random Forest. Quantile regression forest is based on the alternative interpretation of Random Forest as an approximate nearest neighbor method \cite{lin2006random}, the details of which will be reviewed in Section 3.1. 

In recent years, motivated by the impressive successes of Deep Neural Network (DNN) models in miscellaneous machine learning tasks, there has been growing interest in enhancing DNN algorithms with uncertainty estimation capabilities. For instance, Mean Variance Estimation (MVE) \cite{khosravi2014optimized,lakshminarayanan2017simple} assumes the conditional ${{D}_{\left. Y \right|X=x}}$ to be Gaussian, and jointly predicts its mean and variance via maximum likelihood estimation (MLE), while \cite{gal2016dropout} propose using Monte Carlo dropout \cite{srivastava2014dropout} to estimate predictive uncertainty. We refer the interested reader to \cite{pmlr-v80-pearce18a} for a up-to-date overview of related techniques. Despite the differences in learning principles, most existing DNN-based approaches still adopt the traditional strategy of extracting quantile-based equal-tailed intervals. 

\subsection{Quality-based PI Estimation}

As demonstrated in the previous section, existing studies related to PI construction have tended to focus on the intermediate problems of estimating conditional distributions and quantiles, and yet very little attention has been paid to quantitatively examine the quality of elicited PIs \cite{khosravi2011lower}. The coverage criteria (\ref{PI}) could be adopted to evaluate the constructed intervals, but (\ref{PI}) alone is not sufficient for determining a meaningful PI. For example, if the intervals are set to be wide enough (e.g., $\left[ -\infty ,+\infty  \right]$), the true response values could always be contained therein. This phenomenon emphasizes a fundamental trade-off between coverage probability and width of the PI \cite{khosravi2011comprehensive}, and motivates the following quality-based criteria for determining optimal PIs \cite{pmlr-v80-pearce18a}:
\begin{equation}
\label{PIwithBudget}
\begin{aligned}
& \underset{l,u}{\mathop{\min }}\enspace u-l \\ 
& \text{s}\text{.t}\text{.}\enspace P\left( l\le Y\le u\left| X=x \right. \right)\ge 1-\alpha.  \\ 
\end{aligned}
\end{equation}
In other words, given any $0\le \alpha \le 1$, we would like to find the shortest interval that covers the required probability mass. Such intervals are known as the highest density intervals in the statistical literature \cite{box2011bayesian}. An illustrative example that compares highest density and traditional quantile-based equal-tailed intervals is provided in Fig.1.
\begin{figure}
	\centering
	\includegraphics[width=6cm]{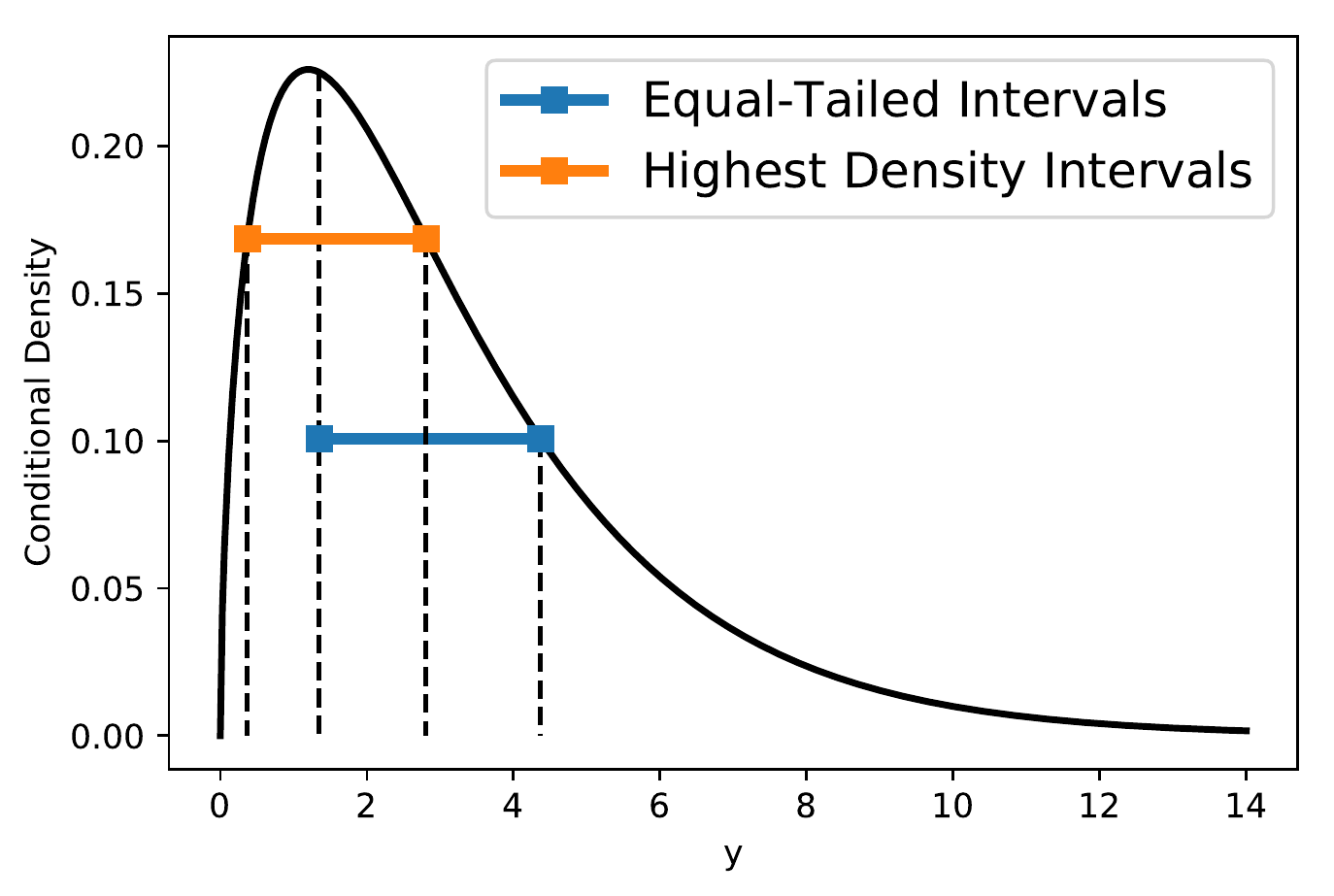} 
	\caption{Comparison of 50\% highest density and equal-tailed intervals for a Gamma distribution $\Gamma \left( 0,2 \right)$.}
	\label{HDI}
\end{figure}

So far, a number of methods have been proposed to predict quality-based PIs. The key idea shared by these approaches is to infer model parameters by minimizing a loss function based on (\ref{PIwithBudget}). The loss function typically consists of two parts, which respectively measures the mean PI width (MPIW) and PI coverage probability (PICP) on the training set \cite{pmlr-v80-pearce18a}. For example, Lower Upper Bound Estimation (LUBE) \cite{khosravi2011lower} trains neural networks by optimizing 
\begin{equation}
	\label{LUBE}
	\frac{\text{MPIW}}{r}\left( 1+\exp \left( \lambda \max \left( 0,\left( 1-\alpha  \right)-\text{PICP} \right) \right) \right),
\end{equation}
where $r$ is the numerical range of the response variable, $\lambda$ controls the trade-off between PICP and MPIW. A limitation of the LUBE loss (\ref{LUBE}) is that it is non-differentiable and hard to optimize. Quality-driven Neural Network Ensemble (QD-Ens) \cite{pmlr-v80-pearce18a} instead minimizes
\begin{equation}
	\label{QD-Ens}
	\text{MPI}{{\text{W}}_{\text{capt}}}+\lambda \frac{n}{\alpha \left( 1-\alpha  \right)}\max {{\left( 0,\left( 1-\alpha  \right)-\text{PICP} \right)}^{2}},
\end{equation}
where \textit{n} is the training set size, $\text{MPI}{{\text{W}}_{\text{capt}}}$ denotes MPIW of points that fall into the predicted intervals. Experiments show that QD-Ens significantly outperforms previous neural-network-based methods for eliciting quality-based PIs. On the other hand, \cite{rosenfeld2018} propose IntPred for quality-based PI construction in the batch learning setting, where PIs for a set of test points are constructed simultaneously. To deal with the non-differentiability of PI quality metrics, both QD-Ens and IntPred instead adopt proxy losses that can be minimized efficiently using standard optimization techniques.

The proposed HDI-Forest model is different from existing quality-based PI estimation works in multiple aspects. Firstly, existing methods mainly consider linear or DNN-based predictive functions, while HDI-Forest is built upon tree ensembles; secondly, HDI-Forest does not require local search heuristics or smooth/convex loss relaxation techniques, and could efficiently obtain global optimal solution of the non-differentiable objective function; finally, by exploiting the special property of Random Forest that will be discussed in the next section, HDI-Forest does not require model re-training for different trade-offs between PICP and MPIW.  

\section{Random Forest}

%We briefly introduce Random Forest, which are the basis of our proposed method. 
A Random Forest is a predictor consisting of a collection of \textit{m} randomized regression trees. Following the notations in \cite{breiman2001random,meinshausen2006quantile}, each tree $T\left( \theta  \right)$ in this collection is constructed based on a random parameter vector $\theta$. In practice, $\theta$ could control various aspects of the tree growing process, such as the re-sampling of the input training set and the successive selections of variables for tree splitting. Once learned, the \textit{L} leaves of $T\left( \theta  \right)$ partition the input feature space $\mathcal{X}$ into \textit{L} non-overlapping axis-parallel subspaces $\left\{ {{\mathcal{X}}_{1}},{{\mathcal{X}}_{2}},\cdots ,{{\mathcal{X}}_{L}} \right\}$, then for any $x\in \mathcal{X}$, there exists one and only one leaf $l\left( x,\theta  \right)$ such that $x\in {{\mathcal{X}}_{l\left( x,\theta  \right)}}$. Meanwhile, the prediction of  $T\left( \theta  \right)$ for \textit{x} is given by averaging over the observations that fall into ${{\mathcal{X}}_{l\left( x,\theta  \right)}}$. Concretely, let $w\left( {{x}_{i}},x,\theta  \right),1\le i\le m$ be defined as:
\begin{equation}
\label{weighted_mean}
w\left( {{x}_{i}},x,\theta  \right)=\frac{\mathbb{I}\left( {{x}_{i}}\in {{\mathcal{X}}_{l\left( x,\theta  \right)}} \right)}{\left| \left\{ j:{{x}_{j}}\in {{\mathcal{X}}_{l\left( x,\theta  \right)}} \right\} \right|},
\end{equation}
where $\mathbb{I}\left( \cdot  \right)$ is the indicator function and  $\left| \cdot  \right|$ denotes the cardinality of a set, then 
\begin{equation}
{{\widehat{y}}_{\text{single-tree}}}\left( x,\theta  \right)=\sum\limits_{i=1}^{n}{w\left( {{x}_{i}},x,\theta  \right){{y}_{i}}}.
\end{equation}
In Random Forest regression, the conditional mean of \textit{Y} given $X=x$ is predicted as the average of predictions of \textit{m} trees constructed with i.i.d. parameters ${{\theta }_{i}}$, $1\le i\le m$:
\begin{equation}
\label{prediction}
\widehat{y}\left( x \right)=\frac{1}{m}\sum\limits_{i=1}^{m}{{{\widehat{y}}_{\text{single-tree}}}\left( x,{{\theta }_{i}} \right)}.
\end{equation}

\subsection{Random Forest for Conditional Distribution Estimation}

Although the original formulation of Random Forest only predicts the conditional mean, the learned trees could also be exploited to predict other interesting quantities \cite{meinshausen2006quantile,li2017forest,feng2018autoencoder}. For example, note that (\ref{prediction}) can be rearranged as 
\begin{equation}
\widehat{y}\left( x \right)=\sum\limits_{i=1}^{n}{w\left( {{x}_{i}},x \right){{y}_{i}}},
\end{equation}
where 
\begin{equation}
\label{predictionrefor}
w\left( {{x}_{i}},x \right)=\frac{1}{m}\sum\limits_{i=1}^{m}{w\left( {{x}_{i}},x,{{\theta }_{i}} \right)}.
\end{equation}
Therefore, (\ref{prediction}) could be alternatively interpreted as the weighted average of the response values of all training instances, and the weight for a specific instances ${{x}_{i}}$ measures the frequency that ${{x}_{i}}$ and \textit{x} are partitioned into the same leaf in all grown trees, which offers an intuitive measure of similarity between them. Theoretically, it can also be shown that ${{x}_{i}}$ tends to be weighted higher if the conditional distributions ${{D}_{\left. Y \right|X=x}}$ and ${{D}_{\left. Y \right|X={{x}_{i}}}}$ are similar \cite{lin2006random}. Furthermore, note that the conditional cumulative distribution function of $Y$ given $X=x$ can be written as:
\begin{equation}
\label{cdf}
\begin{aligned}
F\left( \left. y \right|X=x \right)&=P\left( \left. Y\le y \right|X=x \right) \\ 
& =E\left( \left. \mathbb{I}\left( Y\le y \right) \right|X=x \right).
\end{aligned}
\end{equation}
It could be proven that under certain conditions, (\ref{cdf}) can be estimated using the weights from (\ref{predictionrefor}) as \cite{meinshausen2006quantile}:
\begin{equation}
\label{pred-cdf}
\widehat{F}\left( \left. y \right|X=x \right)=\sum\limits_{i=1}^{n}{w\left( {{x}_{i}},x \right)\mathbb{I}\left( {{y}_{i}}\le y \right)}.
\end{equation}
It has been demonstrated in \cite{meinshausen2006quantile} that (\ref{pred-cdf}) can be exploited to accurately estimate conditional quantiles. In this work, we instead utilize it to perform quality-based PI estimation, as detailed in the next section. 

%In this section, we describe the proposed high-density interval regression forest model. Before we proceed further, it is important to note that there exists a fundamental trade-off between accuracy (which measures the probability that the PIs would cover the true response values) and interval width when eliciting PIs, as the accuracy can be arbitrarily high simply by making the intervals large enough to always contain the true response values, while large intervals are less meaningful and are not suitable for our considered application of picking out hotels that closely match travelers' price expectations.

\section{Highest Density Interval Regression Forest}

In the section, we describe the proposed HDI-Forest algorithm. Concretely, we first use the standard Random Forest algorithm to infer a number of trees from the data, then based on (\ref{pred-cdf}), for any observation \textit{x}, the probability that the associated response value would fall into interval $[l,u]$ can be estimated as:
\begin{equation}
\label{intervalprob}
\widehat{P}\left( \left. l\le Y\le u \right|X=x \right)=\sum\limits_{i=1}^{n}{w\left( {{x}_{i}},x \right) \mathbb{I}\left( l\le {{y}_{i}}\le u \right)}.
\end{equation}
Using (\ref{intervalprob}), the quality-based criteria (\ref{PIwithBudget}) can be approximated as the following optimization problem:
\begin{equation}
\label{obj}
\begin{aligned}
& \underset{l,u}{\mathop{\min }}\enspace u-l \\ 
& \text{s}\text{.t}\text{.}\enspace \widehat{P}\left( l\le Y\le u\left| X=x \right. \right)\ge 1-\alpha.  \\ 
\end{aligned}
\end{equation}
Note that the optimization problem in (\ref{obj}) is non-convex since its constraint function is piece-wise constant and discontinuous. However, its global optimal solution can still be efficiently obtained by exploiting the problem structure, as detailed below.

Firstly, we present Theorem 1, which shows that the optimal solution of (\ref{obj}) must exist in a pre-defined finite set:
\begin{theorem}\label{thm1}
	The optimal solution of (\ref{obj}) satisfies the following conditions:
	\begin{equation}
	\label{opt_cond}
	u, l\in \left\{ {{y}_{i}} \right\}_{i=1}^{n}.
	\end{equation}	
\end{theorem}
\begin{proof} 
	Assume by contradiction that a pair of $[l,u]$ optimizes (\ref{obj}) and does not satisfy (\ref{opt_cond}), then
	\begin{equation}
	\label{not_hold}
	\widehat{P}\left( \left. l\le Y\le u \right|X=x \right)\ge 1-\alpha.
	\end{equation}
	Let ${{l}_{\text{alt}}}$ and ${{u}_{\text{alt}}}$ be defined as
	\begin{equation}
	\label{l_alter}
	{{l}_{{\text{alt}}}}=\underset{i}{\mathop{\min }}\,\left\{ \left. {{y}_{i}} \right|{{y}_{i}}\ge l,1\le i\le n \right\},
	\end{equation}
	\begin{equation}
	\label{r_alter}
	{{u}_{\text{alt}}}=\underset{i}{\mathop{\max }}\,\left\{ \left. {{y}_{i}} \right|{{y}_{i}}<u,1\le i\le n \right\}.
	\end{equation}
	Recall that $[l,u]$ does not satisfy (\ref{opt_cond}), thus either ${{u}_{\text{alt}}}\ne u$ or ${{l}_{\text{alt}}}\ne l$, and
	\begin{equation}
	\label{cond_sub_1}
	{{u}_{{\text{alt}}}}-{{l}_{{\text{alt}}}}<u-l.
	\end{equation}
	Meanwhile, by combining (\ref{not_hold}), (\ref{l_alter}), and (\ref{r_alter}), we have
	\begin{equation}
	\label{cond_sub_2}
	\begin{aligned}
	& \widehat{P}\left( \left. {{l}_{{\text{alt}}}}\le Y\le {{u}_{{\text{alt}}}} \right|X=x \right) \\ 
	& =\sum\limits_{i=1}^{n}{w\left( {{x}_{i}},x \right)\mathbb{I}\left( {{l}_{{\text{alt}}}}\le {{y}_{i}}\le {{u}_{{\text{alt}}}} \right)} \\ 
	& =\sum\limits_{i=1}^{n}{w\left( {{x}_{i}},x \right)\mathbb{I}\left( {{l}}\le {{y}_{i}}\le {{u}} \right)} \\ 
	& =\widehat{P}\left( \left. l\le Y\le u \right|X=x \right) \\ 
	& \ge 1-\alpha .  
	\end{aligned}
	\end{equation}
	Equations (\ref{cond_sub_1}) and (\ref{cond_sub_2}) mean that $[{{l}_{\text{alt}}}, {{u}_{\text{alt}}}]$ is a feasible and better solution than $[l,u]$, which contradicts the assumption that $[l,u]$ optimizes (\ref{obj}).
\end{proof} 
Let the unique elements of $\left\{ {{y}_{i}} \right\}_{i=1}^{n}$ be arranged in increasing order as ${{\widetilde{y}}_{1}}<{{\widetilde{y}}_{2}}<\cdots <{{\widetilde{y}}_{\widetilde{n}}}$. Then based on Theorem 1, (\ref{obj}) can be equivalently reformulated as 
\begin{equation}
\label{new_formula}
\begin{aligned}
& \underset{i,j}{\mathop{\min }}\enspace {{\widetilde{y}}_{j}}-{{\widetilde{y}}_{i}} \\ 
& \text{s}\text{.t}\text{.}\enspace \sum\limits_{k=i}^{j}{{{w}_{k}}}\ge 1-\alpha , \\ 
\end{aligned}
\end{equation}
where 
\begin{equation}
	\label{wk}
	{{w}_{k}}=\sum\limits_{i=1}^{n}{\mathbb{I}\left( {{y}_{i}}={{\widetilde{y}}_{k}} \right)}w\left( {{x}_{i}},x \right).
\end{equation}
Problem (\ref{new_formula}) can then be solved simply by enumerating and evaluating all pairs of elements from $\left\{ {{\widetilde{y}}_{i}} \right\}_{i=1}^{\widetilde{n}}$, which nevertheless is still costly and takes $O\left( {{\widetilde{n}}^{2}} \right)$ time per prediction. Fortunately, we can reduce the time complexity by rearranging the computations, so that the time is only linear to $O\left( \widetilde{n} \right)$. The method is described below.

Firstly, (\ref{new_formula}) can be optimized using a two-stage approach instead: we start by solving the following optimization problem for each $1\le i\le \widetilde{n}$:

\begin{equation}
\label{subproblem}
\begin{aligned}
& \underset{j}{\mathop{\min }}\enspace {{\widetilde{y}}_{j}}-{{\widetilde{y}}_{i}} \\ 
& \text{s}\text{.t}\text{.}\enspace \sum\limits_{k=i}^{j}{{{w}_{k}}}\ge 1-\alpha. \\ 
\end{aligned}
\end{equation}
Then, let $\mathcal{I}\subseteq \left\{ \left. i \right|i=1,2,\cdots ,\widetilde{n} \right\}$ be the set of indices for which (\ref{subproblem}) has a feasible solution, and the optimal solution of (\ref{subproblem}) for $i\in \mathcal{I}$ be denoted as ${{j}_{{\text{opt}}}}\left( i \right)$, it is easy to verify that (\ref{new_formula}) is optimized by the pair of $\left( i,{{j}_{\text{opt}}}\left( i \right) \right)$ that attains the smallest ${{\widetilde{y}}_{{{j}_{\text{opt}}}\left( i \right)}}-{{\widetilde{y}}_{i}}$.

To compute ${{j}_{{\text{opt}}}}\left( i \right)$ for $i\in \mathcal{I}$, we exploit the strict monotonicity of $\left\{ {{\widetilde{y}}_{i}} \right\}_{i=1}^{\widetilde{n}}$, and equivalently reformulate (\ref{subproblem}) as
\begin{equation}
	\label{subproblem1}
	\begin{aligned}
	& \underset{j}{\mathop{\min }}\enspace j \\ 
	& \text{s}\text{.t}\text{.}\enspace \sum\limits_{k=i}^{j}{{{w}_{k}}}\ge 1-\alpha. \\ 
	\end{aligned}	
\end{equation}
In other words, ${{j}_{{\text{opt}}}}\left( i \right)$ is simply the smallest index for which the constraint in  (\ref{subproblem}) holds. Moreover, ${{j}_{{\text{opt}}}}\left( i \right)$ is monotonously increasing with respect to $i$:
\begin{theorem}\label{thm2}
	For any $1\le {{i}_{2}}<{{i}_{1}}\le \widetilde{n}$, we have
	\begin{equation}
	\label{opt_cond_1}
	{{j}_{\text{opt}}}\left( {{i}_{2}} \right)\le {{j}_{\text{opt}}}\left( {{i}_{1}} \right).
	\end{equation}
\end{theorem}	
\begin{proof}
	Assume by contradiction that there exist ${{i}_{1}}$ and ${{i}_{2}}$ such that ${{j}_{\text{opt}}}\left( {{i}_{2}} \right)>{{j}_{\text{opt}}}\left( {{i}_{1}} \right)$ and ${{i}_{1}}>{{i}_{2}}$, recall from the definitions in (\ref{weighted_mean}), (\ref{predictionrefor}) and (\ref{wk}) that ${{w}_{k}}\ge 0,1\le k\le \widetilde{n}$, therefore
	\begin{equation}
	\label{contra_cond_1}
	\begin{aligned}
	\sum\limits_{k={{i}_{2}}}^{{{j}_{{\text{opt}}}}\left( {{i}_{1}} \right)}{w_k}&=\sum\limits_{k={{i}_{1}}}^{{{j}_{{\text{opt}}}}\left( {{i}_{1}} \right)}{w_k}+\sum\limits_{k={{i}_{2}}}^{{{i}_{1}}-1}{w_k} \\ 
	& \ge \sum\limits_{k={{i}_{1}}}^{{{j}_{{\text{opt}}}}\left( {{i}_{1}} \right)}{w_k} \\ 
	& \ge 1-\alpha.
	\end{aligned}		
	\end{equation}
	On the other hand, based on the strict monotonicity of $\left\{ {{\widetilde{y}}_{i}} \right\}_{i=1}^{\widetilde{n}}$ we have
	\begin{equation}
	\label{contra_cond_2}
	{{\widetilde{y}}_{{{j}_{\text{opt}}}\left( {{i}_{1}} \right)}}-{{\widetilde{y}}_{{{i}_{2}}}}<{{y}_{{{j}_{\text{opt}}}\left( {{i}_{2}} \right)}}-{{\widetilde{y}}_{{{i}_{2}}}}.
	\end{equation}
	Equations (\ref{contra_cond_1}) and (\ref{contra_cond_2}) contradict the optimality of ${{{{j}_{{\text{opt}}}}\left( {{i}_{2}} \right)}}$ and thereby we complete the proof.
\end{proof}
Based on the above analysis, in order to solve (\ref{subproblem}) for all $1\le i\le \widetilde{n}$, we only need to walk down the sorted list of {$\left\{ {{\widetilde{y}}_{i}} \right\}_{i=1}^{\widetilde{n}}$ once to identify for each \textit{i} the first index \textit{j} such that $\sum\limits_{k=i}^{j}{{{w}_{k}}}\ge 1-\alpha $. The whole algorithm is presented in Algorithm 1. %The whole algorithm is presented in Algorithm 1 and its overall complexity is dominated by the complexity of sorting the {$\left\{ {{y}_{i}} \right\}_{i=1}^{n}$, which is $O\left( n\log n \right)$. However, this operation only needs to be performed once to handle all future predictions, thus would only incur negligible additional computational overhead. 

\begin{algorithm}
	\caption{Solve (\ref{subproblem}) for all $1\le i\le \widetilde{n}$}
	\begin{algorithmic}[1]
		\Require {$\left\{ {{\widetilde{y}}_{i}} \right\}_{i=1}^{\widetilde{n}}$ sorted in increasing order, the associated weights $\left\{ {{w}_{i}} \right\}_{i=1}^{\widetilde{n}}$, threshold $\alpha$}
		\State {$i\gets 1$, $j\gets 0$, $w\gets 0$}
		\For{$1\le i\le n$} \Comment{Outer loop}
		\While{$j\le n$ AND $w<\alpha $}  \Comment{Inner loop}
		\State {$j\gets j+1$, $w\gets w+{w}_{j}$}
	%	\State {$w\gets w+{w}_{j}$}
		\EndWhile
		\If {$w\ge \alpha $} \Comment{(\ref{subproblem}) has a feasible solution}
		\State ${{j}_{\text{opt}}}\left( i \right)=j$
		\EndIf	
		\State{$w\gets w-{w}_{i}$}	
		\EndFor
	\end{algorithmic}
\end{algorithm}	

\section{Experiments}

\subsection{Experimental Settings}

\subsubsection{Baseline Methods}

Based on the survey of related works in Section 2, we adopted two types of baseline methods for comparison, including quantile-based methods and quality-based methods. Quantile-based methods include Quantile Regression Forest (QRF) \cite{meinshausen2006quantile}\footnote{https://cran.r-project.org/web/packages/quantregForest/index.html}, Quantile Regression (QR) \cite{koenker2001quantile}\footnote{https://cran.r-project.org/web/packages/quantreg/index.html}, and Gradient Boosting Decision Tree with Quantile Loss ($\text{QR}_{\text{GBDT}}$) implemented in the Scikit-learn package \cite{pedregosa2011scikit}. On the other hand, quality-based PI methods include IntPred \cite{rosenfeld2018} and Quality-Driven Ensemble (QD-Ens) \cite{pmlr-v80-pearce18a}\footnote{https://github.com/TeaPearce/}, the state-of-the-art approach for neural-network-based PI elicitation.

\subsubsection{Benchmark Datasets}

We compare various methods on 11 datasets from the UCI repository\footnote{http://archive.ics.uci.edu/ml/index.php}. Statistics of these datasets are presented in Table 1. Each dataset is split in train and test sets according to a 80\%-20\% scheme, and we report the average performance over 10 random data splits. The hyper-parameters of all tested methods were tuned via 5-fold cross-validation on the training set. % We report average errors (1-accuracy) over 100 random data splits. For datasets with less than 25 features, we augment each example with pairwise terms. 

\begin{table}[htbp]
	\footnotesize
	%\normalsize
	%\begin{adjustbox}{width=\columnwidth}
	\centering
	\begin{tabular}{c|cc}
		\hline
		\ Dataset   & Size   & Dimensionality \\\hline
		\multirow{1}*{Boston Housing}
		& 506  & 13 \\
		\multirow{1}*{Parkinsons}
		& 5875  & 26 \\
		\multirow{1}*{Wine Quality}
		& 1599 & 11 \\
		\multirow{1}*{Forest Fires}
		& 517 & 13 \\
		\multirow{1}*{Concrete Compression Strength}
		& 1030 & 9 \\
		\multirow{1}*{Energy Efficiency}
		& 768 & 8 \\
		\multirow{1}*{Naval Propulsion}
		& 11934 & 16 \\
		\multirow{1}*{Combined Cycle Power Plant}
		& 9568 & 4 \\
		\multirow{1}*{Protein Structure}
		& 45730 & 9 \\
		\multirow{1}*{Communities}
		& 1994 & 128 \\
		\multirow{1}*{Online News Popularity}
		& 39797 & 61 \\
		\hline
	\end{tabular}		
	%\end{adjustbox}
	\caption{Characteristics of the datasets used in the experiments.}
	\label{tab:data}
\end{table}

\subsubsection{Evaluation Metrics}

Following previous works \cite{khosravi2011lower,pmlr-v80-pearce18a,rosenfeld2018}, PICP and MPIW mentioned in Section 2 were adopted as the evaluation metrics.

\subsection{Performance Comparison}

As mentioned earlier, there exists a trade-off between the coverage probability (measured by the PICP metric) and width (measured by the MPIW metric) of extracted PIs. To facilitate the comparison of various methods, we first evaluate their performance when they achieve roughly the same level of PICP. The results are presented in Table 2. As can be seen, HDI-Forest significantly outperforms all other baselines for all but one dataset. Compared with the best-performing baseline method (QD-Ens), HDI-Forest could substantially reduce the average interval width by 34\%, while achieving slightly better coverage probability.

To further compare the performance of HDI-Forest against QD-Ens and QRF, the two top-performing baselines, we examine the MPIW scores of three methods for a range of PICP values. As shown in Fig.2, HDI-Forest still achieves the best performance among all models.

\begin{table*}[tb]
	\normalsize
	%\scriptsize
	\footnotesize
	\centering
	%\begin{adjustbox}{width=160mm}
	%\centering
	\begin{tabular}{c|p{15mm}<{\centering}|@{\hspace*{0.12cm}}p{18mm}<{\centering}@{\hspace*{0.12cm}}p{15mm}<{\centering}@{\hspace*{0.12cm}}p{15mm}<{\centering}@{\hspace*{0.12cm}}p{15mm}<{\centering}@{\hspace*{0.12cm}}p{15mm}<{\centering}@{\hspace*{0.12cm}}p{15mm}<{\centering}@{\hspace*{0.12cm}}p{15mm}<{\centering}}
		\hline
		\ Dataset  & Metrics & HDI-Forest & QRF & QR & $\text{QR}_{\text{GBDT}}$ & IntPred & QD-Ens \\\hline
		\multirow{2}*{Boston Housing}
		& PICP & \textbf{0.93}   & 0.92 & 0.92 & 0.92 & 0.92 & 0.91 \\
		& MPIW & \textbf{1.02}   & 1.19 & 2.28 & 1.70 & 1.79 & 1.16 \\\hline
		\multirow{2}*{Parkinsons}
		& PICP & \textbf{0.99}   & 0.99 & 0.95 & 0.99 & 0.96 & 0.98 \\
		& MPIW & \textbf{0.11}   & 0.45 & 1.18 & 0.88 & 0.98 & 0.62 \\\hline	
		\multirow{2}*{Wine Quality}
		& PICP & \textbf{0.92}   & 0.92 & 0.91 & 0.92 & 0.92 & 0.92 \\
		& MPIW & \textbf{1.11}   & 3.33 & 2.64 & 2.47 & 2.54 & 2.33 \\\hline	
		\multirow{2}*{Forest Fires}
		& PICP & \textbf{0.95}   & 0.94 & 0.95 & 0.94 & 0.94 & 0.94 \\
		& MPIW & \textbf{0.81}   & 1.24 & 1.09 & 1.04 & 1.03 & 0.96 \\\hline
		\multirow{2}*{Concrete Compression Strength}
		& PICP & 0.94   & 0.94 & 0.94 & 0.94 & 0.94 & \textbf{0.94} \\
		& MPIW & 1.18   & 1.22 & 2.22 & 2.23 & 1.87 & \textbf{1.09} \\\hline
		\multirow{2}*{Energy Efficiency}
		& PICP & \textbf{0.97}   & 0.96 & 0.95 & 0.97 & 0.95 & 0.97 \\
		& MPIW & \textbf{0.39}   & 0.50 & 1.73 & 0.79 & 1.56 & 0.47 \\\hline
		\multirow{2}*{Naval Propulsion}
		& PICP & \textbf{0.99}   & 0.96 & 0.98 & 0.98 & 0.98 & 0.98 \\
		& MPIW & \textbf{0.24}   & 0.68 & 0.89 & 1.34 & 0.73 & 0.28 \\\hline
		\multirow{2}*{Combined Cycle Power Plant}
		& PICP & \textbf{0.95}   & 0.95 & 0.95 & 0.95 & 0.95 & 0.95 \\
		& MPIW & \textbf{0.75}   & 0.78 & 0.97 & 0.90 & 0.84 & 0.86 \\\hline
		\multirow{2}*{Protein Structure}
		& PICP & \textbf{0.95}   & 0.94 & 0.95 & 0.95 & 0.95 & 0.95 \\
		& MPIW & \textbf{1.77}   & 1.82 & 2.76 & 2.36 & 2.15 & 2.27 \\\hline
		\multirow{2}*{Communities}
		& PICP & \textbf{0.92}   & 0.92 & 0.89 & 0.92 & 0.92 & 0.87 \\
		& MPIW & \textbf{1.50}   & 1.73 & 2.03 & 1.69 & 1.94 & 1.74 \\\hline
		\multirow{2}*{Online News Popularity}
		& PICP & \textbf{0.96}   & 0.96 & 0.95 & 0.96 & 0.95 & 0.96 \\
		& MPIW & \textbf{1.18}   & 1.98 & 1.27 & 1.72 & 1.38 & 1.60 \\\hline \hline	 
		\multirow{2}*{Average Performance}
		& PICP & \textbf{0.95}   & 0.95 & 0.94 & 0.95 & 0.94 & 0.94 \\
		& MPIW & \textbf{0.91}   & 1.35 & 1.73 & 1.56 & 1.53 & 1.22 \\\hline	
	\end{tabular}
	
	%\end{adjustbox}
	\caption{Performance comparison of various methods. The results are averaged over 20 random runs, with best results in bold. Here, the best was chosen according to the strict rule that its performance should be equal to or better than all other methods measured by both PICP and MPIW. The last two rows show the average performance of all tested methods.}
	\label{baseline}
\end{table*}

\begin{figure}
	\centering
	\includegraphics[width=8.4cm]{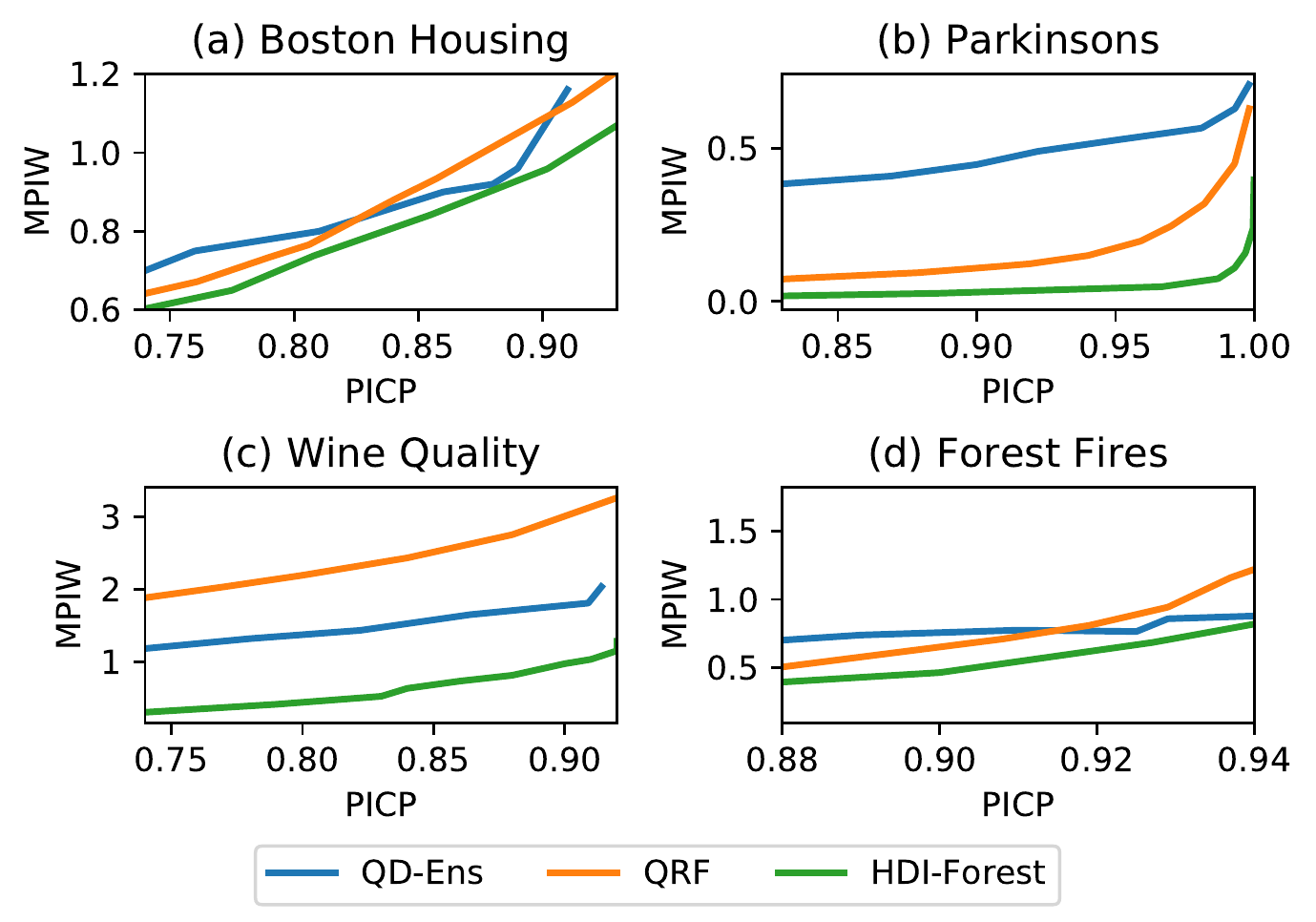} 
	\caption{Comparison between three methods by examining the MPIW score as a function of PICP.}
	\label{Tradeoff}
\end{figure}

\section{Conclusion}

In this paper, we propose HDI-Forest, a novel algorithm for quality-based PI estimation, extensive experiments on benchmark datasets show that HDI-Forest significantly outperforms previous approaches.

For future work, we plan to extend HDI-Forest to the batch learning setting, where the overall performance on a group of test instances can be further improved by adjusting the per-instance coverage probability constraints \cite{rosenfeld2018}. On the other hand, HDI-Forest is based on the original Random Forest model that is mainly suitable for standard regression/classification tasks, however, a large number of Random-Forest-based approaches have been proposed in the literature to handle other types of problems \cite{sathe2017similarity,barbieri2016improving}. It would also be interesting to study quality-based PI estimation for these models.

\bibliographystyle{named}
\bibliography{ijcai17}

\end{document}